\documentclass[letterpaper, 10 pt, conference]{ieeeconf}  

\IEEEoverridecommandlockouts                              

\usepackage{cite}
\usepackage{float}
\usepackage{amsfonts}

\usepackage{amsthm}
\usepackage{amssymb}
\usepackage{mathtools}
\usepackage{multirow}
\usepackage{soul}
\usepackage{longtable}
\usepackage{amsmath}
\usepackage[hidelinks]{hyperref}

\usepackage{graphicx}		
\usepackage{subfigure}
\usepackage{booktabs}
\usepackage{subfigure}
\usepackage{xspace}
\usepackage[dvipsnames]{xcolor}
\usepackage[textsize=small,textwidth=0.8in]{todonotes}
\usepackage{textcomp}

\usepackage[capitalize,noabbrev]{cleveref}
\usepackage{titlesec}
\titlespacing\section{0pt}{2pt}{0pt}
\titlespacing\subsection{0pt}{0pt}{0pt}
\titlespacing\bfsection{0pt}{0pt}{0pt}

\theoremstyle{plain}
\newtheorem{theorem}{Theorem}[section]

\theoremstyle{definition}

\theoremstyle{remark}

\newcommand{\bfsection}[1]{\noindent\textbf{#1}}
\newcommand{\algname}{\textsc{ACS}\xspace}
\newcommand{\thmref}[1]{Theorem~\ref{#1}}

\newcommand{\tabref}[1]{Table~\ref{#1}}
\newcommand{\figref}[1]{Fig.~\ref{#1}}
\newcommand{\eqnref}[1]{\text{Eq.}~(\ref{#1})}
\newcommand{\secref}[1]{\S\ref{#1}}
\newcommand{\appref}[1]{Appendix~\ref{#1}}

\title{\LARGE \bf
Safe Reinforcement Learning via Hierarchical \\Adaptive Chance-Constraint Safeguards
}

\author{Zhaorun Chen$^{1}$, Zhuokai Zhao$^{1}$,
Tairan He$^{2}$, Binhao Chen$^{3}$, Xuhao Zhao$^{3}$, Liang Gong$^{3}$ and Chengliang Liu$^{3}$
\thanks{$^{1}$Zhaorun Chen and Zhuokai Zhao are with the Department of Computer Science, University of Chicago, USA.
        {\tt\small \{zhaorun, zhuokai\}@uchicago.edu}
        }%
\thanks{$^{2}$Tairan He is with the Robotics Institute, Carnegie Mellon University, USA. {\tt\small tairanh@andrew.cmu.edu}.}
\thanks{$^{3}$Binhao Chen, Xuhao Zhao, Liang Gong and Chengliang Liu are with the Department of Mechanical Engineering, Shanghai Jiao Tong University
        {\tt\small \{cbh\_mage, rachmaninov, gongliang\_mi, chlliu\}@sjtu.edu.cn}
        }%
}

\begin{document}

\maketitle
\thispagestyle{empty}
\pagestyle{empty}

\begin{abstract}
%
Ensuring safety in Reinforcement Learning (RL), typically framed as a Constrained Markov 
Decision Process (CMDP), is crucial for real-world exploration applications.
%
%
%
Current approaches in handling CMDP struggle to balance optimality and feasibility, as 
direct optimization methods cannot ensure state-wise in-training safety, and 
projection-based methods correct actions inefficiently through lengthy iterations.
%
To address these challenges, we propose Adaptive Chance-constrained Safeguards (ACS), an 
adaptive, model-free safe RL algorithm using the safety recovery rate as a surrogate 
chance constraint to iteratively ensure safety during exploration and after achieving 
convergence.
Theoretical analysis indicates that the relaxed probabilistic constraint sufficiently 
guarantees forward invariance to the safe set. 
%
%
And extensive experiments conducted on both simulated and real-world safety-critical tasks 
demonstrate its effectiveness in enforcing safety (nearly zero-violation) while preserving 
optimality (+23.8\%), robustness, and fast response in stochastic real-world settings.
\end{abstract}

\section{Introduction}\label{sec:introduction}
Reinforcement learning (RL) has demonstrated remarkable success in handling nonlinear 
stochastic control problems with large
uncertainties~\cite{bertsekas2019reinforcement,johannink2019residual}. 
Although solving unconstrained optimization problems in simulations incurs no
safety concerns, ensuring safety during training is crucial in real-world 
applications~\cite{garcia2015comprehensive}. 
However, the inclusion of safety constraints in RL is non-trivial.
First, safety and goal objectives are often competitive~\cite{awasthi2022theory}.
Second, the constrained state-space is usually non-convex~\cite{wen2018constrained}.
And third, ensuring safety via iterative action corrections is time-consuming and 
impractical for safety-critical tasks requiring fast responses.

Numerous studies strive to address these challenges, enhancing safety assurances in RL. 
Early works utilized trust-region~\cite{achiam2017constrained} and fixed penalty 
methods~\cite{bohez2019value} to enforce cumulative constraints satisfaction in 
expectation~\cite{liang2018accelerated, tessler2018reward}. 
However, these methods are sensitive to hyperparameters and often result in policies being 
either too aggressive or too conservative~\cite{liu2021policy}. 
Furthermore, these methods only ensure safe behaviors asymptotically upon the completion of
training, resulting in a gap in safety assurance during the training exploration 
process~\cite{bharadhwaj2020conservative}. 

Recognizing the limitations of above methods in achieving immediate safety during the 
training phase, many works seek to employ hierarchical agents to project task-oriented 
action into a prior~\cite{thananjeyan2020safety}, or learned safe 
region~\cite{zhang2023evaluating} to satisfy state-wise constraints. 
However, these methods usually make additional assumptions like white-box system 
dynamics~\cite{cheng2019end} or default safe controller~\cite{koller2018learning}, which is not always available and strict 
feasibility is not guaranteed~\cite{wen2018constrained}. 
Besides, applying step-wise projection to satisfy instantaneous hard constraints is often 
time-consuming, which hampers the optimality of the task objectives~\cite{zhao2023state}.
%

Alternatively, integrating safety into RL via chance or probabilistic constraints, which
involves unfolding predictions to estimate safety probability for future states, 
presents a compelling advantage.
Chance constraints have received significant attention within both safe 
control~\cite{wang2022myopically, jing2022probabilistic} and RL 
communities~\cite{zhao2023probabilistic, bharadhwaj2020conservative, petsagkourakis2020constrained}. 
However, while these methods mark a step forward, they often fall short in efficiently 
enforcing such constraints~\cite{bharadhwaj2020conservative, eysenbach2017leave}
or necessitate the use of an independent safety probability 
estimator~\cite{gangadhar2022adaptive}.
%
%


Motivated by these challenges, we propose \textbf{A}daptive \textbf{C}hance-constrained 
\textbf{S}afeguards (\textbf{\algname}), an efficient model-free safe RL algorithm that 
models safety recovery rate as a surrogate chance constraint to adaptively guarantee 
safety in exploration and after convergence. 
Unlike existing work~\cite{petsagkourakis2020constrained} that approximates the safety 
critic 
through lengthy Monte Carlo sampling, or~\cite{bharadhwaj2020conservative} that 
learns a conservative policy with relaxed upper-bounds 
in conservative Q-learning~\cite{kumar2020conservative}, \algname directly constrains the safety advantage critics, 
which can be interpreted as \textit{safety recovery rate}. 
We show theoretically in \secref{sec:ACS} that this is a sufficient condition to certify 
in-training safety convergence in expectation. 
The introduction of recovery rate mitigates the objective trade-off commonly encountered 
in safe RL~\cite{garcia2015comprehensive} by encouraging agents to explore risky states 
with more confidence, while enforcing strict recovery to the desired safety threshold. 
%
%
We also validate empirically in \secref{sec:exp} that \algname can find a near-optimal 
policy in tasks with stochastic moving obstacles where almost all other state-of-the-art 
(SOTA) algorithms fail. 

To summarize, the contributions of this paper include: 
(1) proposing adaptive chance-constrained safeguards (\algname), an advantage-based 
algorithm mitigating exploration-safety trade-offs with surrogate probabilistic constraints 
that theoretically certifies safety recovery; 
(2) extensive experiments on various simulated safety-critical tasks 
demonstrating that \algname not only achieves superior safety performance (nearly 
zero in-training violation), but also surpasses SOTA methods in cumulative reward and 
time efficiency with a significant increase (23.8\% $\pm$ 10\%);
and (3) two real-world manipulation experiments 
showing that \algname boosts the success rate by~30\% and reduces safety 
violations by~65\%, while requiring fewer iterations compared to existing methods.
%
    
    

\section{Related Work}\label{sec:works}
\subsection{Safe RL}
Safe RL focuses on algorithms that can learn optimal behaviors while ensuring
safety constraints are met during both the training and deployment 
phases~\cite{garcia2015comprehensive, sutton2018reinforcement}.
%
Existing safe RL methods can be generally divided into three categories:
\subsubsection{End-to-end} 
End-to-end agent augments task objective with safety cost and solves unconstrained 
optimization (or its dual problem) directly 
thereafter~\cite{tessler2018reward, liang2018accelerated}. 
%
For example,~\cite{donti2021dc3} augments safety constraint as $L_2$ 
regularization;~\cite{ma2021learn} adopts an additional network 
to approximate the Lagrange multiplier;
and~\cite{he2023autocost} searches for intrinsic cost to achieve zero-violation performance. 
However, the resulting policies are only asymptotically safe and lack 
in-training safety assurance, while the final convergence on safety constraints is not
guaranteed~\cite{chen2021primal}.

\subsubsection{Direct policy optimization (DPO)} 
Instead of augmenting safety cost into the reward function, DPO methods such 
as~\cite{achiam2017constrained} leverage trust-regions to update task policy 
inside the feasible region. 
More specifically,~\cite{wen2018constrained} refines the sampling distribution by solving a 
constrained cross-entropy problem. 
And~\cite{zhang2020first} confines the safe region via a convex approximation 
to the surrogate constraints with first-order Taylor expansion. 
However, these methods are usually inefficient and are prone to being overly conservative.

\subsubsection{Projection-based methods} 
To ensure strict certification in safety, recent work leverages a hierarchical 
safeguard/shield~\cite{alshiekh2018safe} to project unsafe actions into the safe set. 
Projection can be conducted by various approaches, including 
iterative-sampling~\cite{bharadhwaj2020conservative}, 
gradient-descent~\cite{zhang2023evaluating}, 
quadratic-programming (QP)~\cite{chow2018lyapunov}, 
and control-barrier-function~\cite{cheng2019end}. 
More specifically,~\cite{bharadhwaj2020conservative} proposes
an upper-bounded safety critic 
via CQL~\cite{kumar2020conservative} and iteratively collects samples until a 
conservative action that satisfies the safety constraint is found. 
However, iterative sampling is not time-efficient for safety-critical tasks that require 
immediate responses. 
Similarly,~\cite{selim2022safe, ganai2024iterative} conduct black-box reachability 
analysis to iteratively search for entrance to the feasible set.
On the other hand, safety layer method~\cite{dalal2018safe} parameterizes the system 
dynamics and solves the QP problem with the learned dynamics. 
However, we show in \secref{sec:exp} that these methods fail in tasks with complex cost functions. 
Some other methods seek to achieve zero-violation with a hand-crafted energy function, such 
as ISSA~\cite{zhao2021model}, RL-CBF~\cite{wei2019safe}, 
and ShieldNN~\cite{ferlez2020shieldnn}.
However, these methods require prior knowledge of the task, which is intractable for 
general model-free RL case. 
Compared with existing SOTA methods, our proposed method \algname, effectively 
addresses and surpasses them by tackling two major challenges in this field: (1) balancing the trade-off between task optimality 
and safety feasibility; and (2) efficiently conducting projection.
%
\subsection{Chance-Constrained Safe Control}
A tentative approach for better trading-off between objectives is through unrolling 
future predictions, and deriving a surrogate chance-constraint to prevent future safety 
violations~\cite{wang2022myopically}. 
However, the inclusion of chance-constraints in control optimization is non-trivial.
\cite{gangadhar2022adaptive} proposes to linearize the chance-constraint 
into a myopic controller to guarantee long-term safety. 
However, their method requires a refined system model and an extra differentiable safe 
probability estimator. 
In model-free RL, state-wise safety probability depends on the policy 
and can thus be approximated with a critic~\cite{srinivasan2020learning}.
%
Early works~\cite{petsagkourakis2020constrained, wen2018constrained} approximate this 
critic through Monte Carlo sampling, which are lengthy and slow.

To further address the issue,~\cite{bharadhwaj2020conservative} approximates the critic 
via CQL~\cite{kumar2020conservative} and learns a conservative policy via iterative sampling.
Similarly,~\cite{eysenbach2017leave} proposes to learn an ensemble of critics 
and train both a forward task-oriented policy and a reset goal-conditioned policy 
that kicks in when the agent is in an unsafe state. 
Given that resetting is not always necessary and efficient,~\cite{thananjeyan2021recovery} 
proposes to learn a dedicated policy that recovers unsafe states. 
However, these approaches require an additional recovery policy and are prone to being
overly conservative. 
In this paper, we show that using only one advantage-based safeguarded policy can 
achieve comparable recovery capability, while maintaining high efficiency.

\section{Preliminaries and Problem Formulation}\label{sec:prelim}
\subsection{Markov Decision Process with Safety Constraint}\label{subsec:cmdp}
Let $x_k \in \mathcal{X} \subset \mathbb{R}^{n_x}$, $u_k \in \mathcal{U} \subset 
\mathbb{R}^{n_u}$ be the discrete sample of system state and control input of continuous
time $t$ (i.e., $t=k\Delta t$), where $n_x$ and $n_u$ are the dimension of the state space 
$\mathcal{X}$ and control space $\mathcal{U}$, the partially observable system with 
stochastic disturbances can be essentially represented by a probability distribution, 
that is:
\begin{equation}\label{eq:system}
    x_{k+1} = \mathbf{F}(x_k, u_k, \hat{\epsilon}) + w_k \Rightarrow x_{k+1} 
    \sim P(x_{k+1}|x_k, u_k) 
\end{equation}
where $\mathbf{F}$ denotes the system dynamics under parametric uncertainty, and 
$\hat{\epsilon}$, $w_k$ denote the parametric uncertainties and 
the additive disturbance of the system respectively.
%
%
RL policy seeks to maximize rewards in an infinite-horizon Constrained Markov Decision 
Process (CMDP)~\cite{altman2021constrained}, which can be specified by a tuple 
$(\mathcal{X},\mathcal{U}, \gamma, R, P)$, 
where $R: \mathcal{X} \times \mathcal{U} \rightarrow \mathbb{R}$ is the reward 
function, $0\le \gamma \le 1$ is the discount factor and 
$P: \mathcal{X} \times \mathcal{U} \times \mathcal{X} \rightarrow [0, 1]$ is the 
system state transition probability function defined in \eqnref{eq:system}. 
Therefore, the safe RL problem can be formulated as
%
\begin{subequations}
   \begin{align}\arg\,\max_{\pi_\theta} J(\pi) &= \mathbb{E}_{x \sim P, u \sim \pi_\theta}\left[\sum^{\infty}_{k=0}\gamma^k R(x_k, u_k)\right] \label{eqn:problem_statement}    \\
        \mathrm{ s.t. } \hspace{0.2cm} \pi_\theta &\in \Pi_C \label{eqn:pi_theta}\\
        \Pi_C &= \{\pi \in \Pi \mid \forall u_k \sim \pi_\theta, x_k \in \mathcal{S}_C\} \label{eqn:Pi_theta}\\
        \mathcal{S}_C &= \{x \mid \forall i, J_{C_i}(x)\le d_i\ \} \label{eqn:s_c}
    \end{align} 
\end{subequations}
where $\Pi$ denotes the set of all stationary policies, $\Pi_C \subset \Pi$ represents the
set of \textit{feasible} policies that satisfy all safety constraints 
$C = \{C_0, C_1, \dots, C_n\}$.
Accordingly, $J_{C_i}$ denotes the cost measure with respect to one specific safety 
constraint $C_i \in C$, and is evaluated as
$J_{C_i}=\mathbb{E}_{\tau\sim\pi} [\sum_{k=0}^\infty \gamma^k C_i(x_k, u_k, x_{k+1})]$,
where $\tau$ is a trajectory (i.e., $\tau=\{x_0, u_0, x_1, u_1 \cdots\}$) resulted from 
the policy $\pi$.
Finally, $\mathcal{S}_C$ represents the set of \textit{safe} states, where each state
satisfies the $i^\text{th}$ safety constraints $C_i \in C$ through not exceeding its 
corresponding permitted threshold $d_i$.
%
\subsection{Chance-constrained Safety Probability}\label{subsec:cc_safety}
However, enforcing all the states in a trajectory to stay within $\mathcal{S}_C$ as defined 
in \eqnref{eqn:s_c} is impractical in real-world settings. 
Because, first, the environment is stochastic with large uncertainties, which is impossible 
to satisfy \eqnref{eqn:pi_theta} all the time. 
Second, the penalty feedback induced by the final disastrous behavior is often sparse and 
delayed. Thus solely abiding by these constraints will result in myopic, unrecoverable 
policies~\cite{wang2022myopically}.
%
Therefore, we propose to instead consider a chance constraint~\cite{gangadhar2022adaptive} 
by unrolling future predictions and ensuring $x_k \in \mathcal{S}_C$ during an outlook time 
window $\mathcal{T}(k)=\{k, k+1, \dots\}$ with probability $1-\alpha$.
Precisely, we define the safety probability $\Psi$ of a single state $x_k$ as
\begin{equation}\label{eqn:psi_x_k}
        \Psi(x_k) := P\left(\cap_{j \in \mathcal{T}(k)} \,x_j \in \mathcal{S}_C\right) 
        \geq 1-\alpha
\end{equation}
where $\alpha$ represents the tolerance level of the unsafe event. 

To evaluate the constraint in \eqnref{eqn:psi_x_k}, we draw from the Bellman 
equation~\cite{sutton2018reinforcement} and approximate the expected long-term 
chance-constrained safety probability $\Psi$ with a value network $V_C^\pi(x_k)$ 
that is learned through trials and errors. 
The associated theorem (\thmref{thm:safe_probability}) and its proof are detailed 
in \appref{app:value_approx}.
%

While other works may evaluate the state-action value 
$Q_C^\pi(x_k, u_k)$~\cite{bharadhwaj2020conservative, zhang2023evaluating},
in this paper, we follow~\cite{wagener2021safe} and approximate the 
advantage $A_C^\pi(x_k, u_k)$ of control $u_k$ with a critic network, since it could better
adjust with the change of safety probability, as will be shown later in \secref{subsec:recover}.
%

\section{Adaptive Chance-Constrained Safeguards}\label{sec:ACS}
\begin{figure}[t!]
    \centering
    \includegraphics[width=0.39\textwidth]{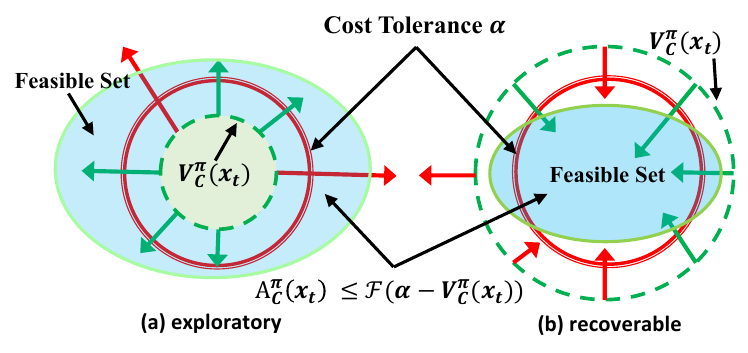}
    \caption{
    The proposed adaptive chance constraint. 
    The green-dashed and red circle denotes the current safety cost and unified cost tolerance level respectively. The blue oval denotes the adaptive chance-constrained feasible set.
    Green/red arrows denote feasible/infeasible actions. 
    When current cost $V^\pi_C(x_k)$ is within tolerance, the agent is encouraged to 
    explore more risky states. 
    Otherwise, the next action is constrained in a more conservative set which 
    satisfies \eqnref{eqn:sufficient_condition}, so that long-term safety recovery is certified.
    %
    }
    \vspace{-0.2in}
    \label{fig:algorithm}
\end{figure}
\begin{figure*}[t!]
    \centering
    \includegraphics[width=0.8\textwidth]{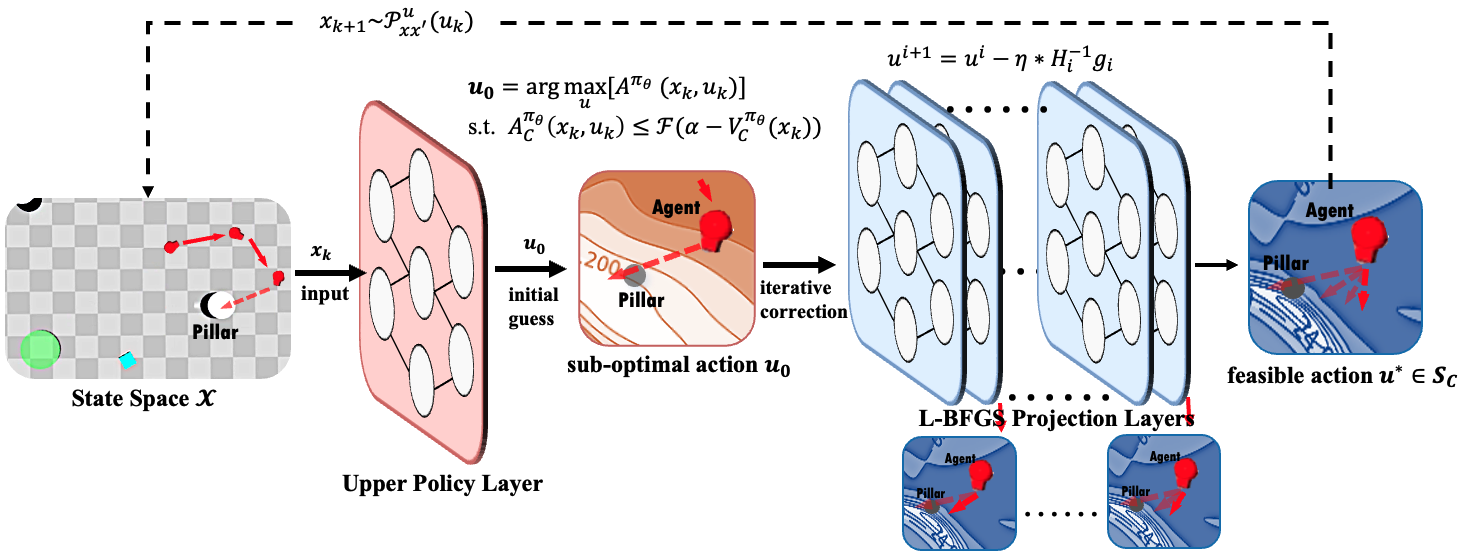}
    \caption{
    The hierarchical framework of the proposed \algname. 
    A Lagrangian-based upper policy layer first generates a near-optimal initial action 
    $u_0$ by solving~\eqnref{eqn:lagrangian}, then the quasi-newton-based projection layers 
    iteratively correct it into the safe set that satisfies \eqnref{eqn:sufficient_condition} 
    via efficient back-propagation \eqnref{eqn:bfgs_update}, enabling ACS to balance task 
    objective and certified safety by constraining actions in an adaptive feasible set while 
    ensuring immediate response.
    }
    \vspace{-0.15in}
    \label{fig:framework}
\end{figure*}
In this section, we illustrate the proposed advantage-based chance-constraint safeguard 
(\algname), which derives a relaxed constraint on RL exploration to achieve better 
task-oriented performance while theoretically guaranteeing recovery to the safe region.
The schematic of \algname is shown in \figref{fig:algorithm}.
%
%
\subsection{Learning to Recover}\label{subsec:recover}
As mentioned earlier in \secref{subsec:cc_safety}, strictly guaranteeing \eqnref{eqn:s_c} 
is neither practical nor necessary for real-world settings. 
To help mitigate the difficulty in specifying a rational safety boundary that 
balances between task and safety objectives, we relax the safety chance 
constraint following~\cite{wang2022myopically} and propose a plug-and-play sufficient 
condition of the safe recovery in ACS for generic RL controllers.
First, we define a discrete-time generator $G$ for any $x_k \in \mathcal{X} \subset 
\mathbb{R}^{n_x}$, which can be considered as a stochastic process~\cite{ethier2009markov}
taking form:
\begin{equation}\label{eqn:discrete_generator_g}
G\Psi(x_k) = \mathbb{E}\left[\Psi(x_{k+1}) \mid x_k, \pi_{\theta}(x_k)\right]-\Psi(x_k)
\end{equation}
where $\pi_{\theta}(x_k)$ is also conditioned on state $x_k$. 
Essentially, \eqnref{eqn:discrete_generator_g} captures the expected improvement 
or degeneration of safety probability $\Psi$ as the stochastic process proceeds.
%
%
However, rather than imposing constraints directly on $\Psi(x_k)$ like 
in~\cite{bharadhwaj2020conservative, zhang2023evaluating},
we propose to apply the chance constraint on the generator output 
(recovery rate of $\Psi(x_k)$), that is:
\begin{equation}\label{eqn:g_psi_x_k}
G\Psi(x_k) \geq -\mathcal{F}\left(\Psi(x_k) - (1-\alpha)\right)
\end{equation}
where $\mathcal{F}(p)$ can be any concave function that is upper-bounded by $p$. 
Consequently, \eqnref{eqn:g_psi_x_k} defines a lower bound for its recovery rate, such that
when $\Psi(x_k) \leq (1-\alpha)$ (i.e. the safety assurance is compromised), the 
controller is enforced to recover to safety at rate $G\Psi(x_k)$. 
Otherwise, the controller is free to explore to achieve better task-oriented performance,
which balances the trade-off between safety and optimality.
%

More specifically, the following chance constraint is proposed in \algname framework to certify 
in-training safety, which is specified in \thmref{thm:constraint}.
The detailed proof of the theorem is illustrated in \appref{app:dynamic_policy}.

\begin{theorem}\label{thm:constraint}
Let $A^{\pi_\theta}_C(x_k, u_k)$ denote the advantage function of control $u_k$ at $x_k$, 
the sufficient condition that can ensure asymptotic safety satisfaction both in training 
and after convergence is
\begin{subequations}\begin{align}\label{eqn:sufficient_condition}
A_{C_i}^{\pi_\theta}(x_k, u_k)  &\leq \mathcal{F}_i(\alpha_i - V_{C_i}^{\pi_\theta}(x_k)) \hspace{0.1cm}  \\
\mathrm{ s.t. } \hspace{0.2cm} H(\mathcal{F}_i(q)) &\preceq 0 \text{ and } \mathcal{F}_i(q) \leq q
\end{align}
\end{subequations}
where $H(\mathcal{F}_i(q))$ is the Hessian of $\mathcal{F}_i(q)$, and $C_i$, $\alpha_i$ 
denotes the cost function and the tolerance level of the $i^\text{th}$ safety constraint, 
respectively.
\end{theorem}

To evaluate the constraint \eqnref{eqn:sufficient_condition} in implementation, we 
construct a fully-connected multi-layer network for both the value and advantage function 
separately based 
on~\cite{stable-baselines3}\footnote{\url{https://github.com/DLR-RM/stable-baselines3}}.
These two networks are updated iteratively together with the task policy in RL exploration through trials-and-errors.

\subsection{Hierarchical Safeguarded Controller}\label{subsec:hierarchical_controller}
In practice, the actual implementation of \algname is illustrated in \figref{fig:framework}. 
To strictly enforce the safety chance constraint, we evaluate the proposed control action 
at every time step and project it into the safe action set~\cite{zhao2023state}. 
Similar to~\cite{donti2021dc3, zhang2023evaluating}, we adopt a hierarchical 
architecture in which the upper policy first solves for a sub-optimal action and iteratively 
corrects it to satisfy the chance constraint in \eqnref{eqn:sufficient_condition}.
%
Specifically, we employ the limited-memory Broyden-Fletcher-Goldfarb-Shanno (L-BFGS) 
method~\cite{liu1989limited} to enforce the feasibility of actions provided by the 
safety-aware policy approximator while preserving time efficiency. 
Note that L-BFGS is theoretically promised to converge much faster to the safety region 
starting from a sub-optimal solution embedded in a locally-convex 
space~\cite{liu1989limited}, as compared to other gradient-descent 
methods~\cite{zhang2023evaluating}.
%
Hence the implementation framework of \algname is consisted of the following 
two sub-modules.

\bfsection{Sub-optimal policy layer.}
For the upper policy layer, we follow~\cite{tessler2018reward} and train a policy optimizer 
to solve task objective \eqnref{eqn:problem_statement} with augmented penalty of the 
chance constraint \eqnref{eqn:sufficient_condition} weighted by a Lagrangian 
multiplier $\lambda$:
\begin{small}
\begin{equation}\label{eqn:lagrangian}
\max_\theta \min_\lambda \mathcal{L}(\pi_\theta, \gamma) 
= \mathbb{E}_{\tau \sim \pi_\theta} \left[A^{\pi_\theta} -\lambda(A_C^{\pi_\theta}- \mathcal{F}(\alpha - V_C^{\pi_\theta}))\right]
\end{equation}
\end{small}

Note that the goal here is to find a safety-aware policy that can produce a sub-optimal 
initial guess $u_0$ with respect to both task objective and safety constraint. 
While extensive existing works solve various forms of 
\eqnref{eqn:lagrangian}~\cite{donti2021dc3,zhang2023evaluating,ma2021learn}, we focus on 
the sub-optimality of the initial guess rather than strictly solving the 
cumulative-constrained MDP. 
Thus, any model-free safety-aware RL algorithm, such 
as~\cite{achiam2017constrained, bharadhwaj2020conservative}, can serve as the policy solver 
of our hierarchical framework. 
We show empirically in \secref{sec:exp} that simply optimizing \eqnref{eqn:lagrangian} is 
efficient enough to find a near-optimal solution via ACS.
\begin{figure*}[t]
    \centering
    \begin{minipage}[t]{0.19\textwidth}
        \centering
        \subfigure[Ant-Run: a quadrupedal robot seeks to run within a velocity limit.]{
            \includegraphics[width=0.86\textwidth]{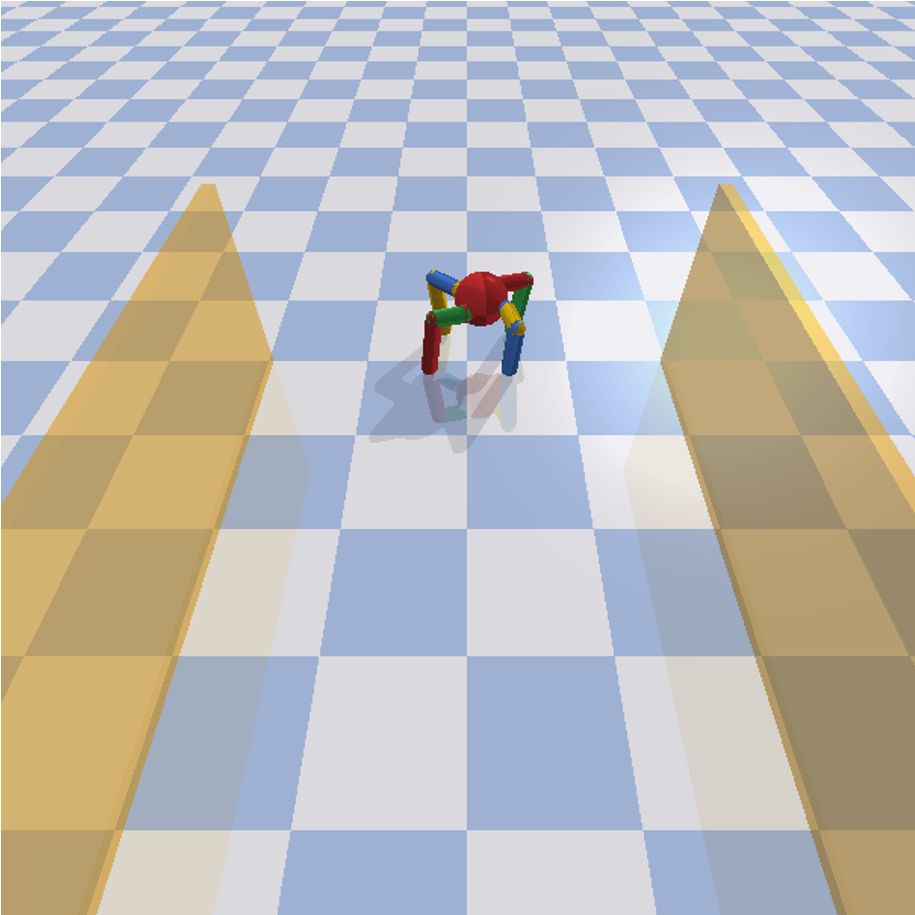}
        }
    \end{minipage}
    \hfill
    \begin{minipage}[t]{0.19\textwidth}
        \centering
        \subfigure[Kuka-Reach: a 7-DOF manipulator Kuka navigates to the button collision-free.]{
            \includegraphics[width=0.86\textwidth]{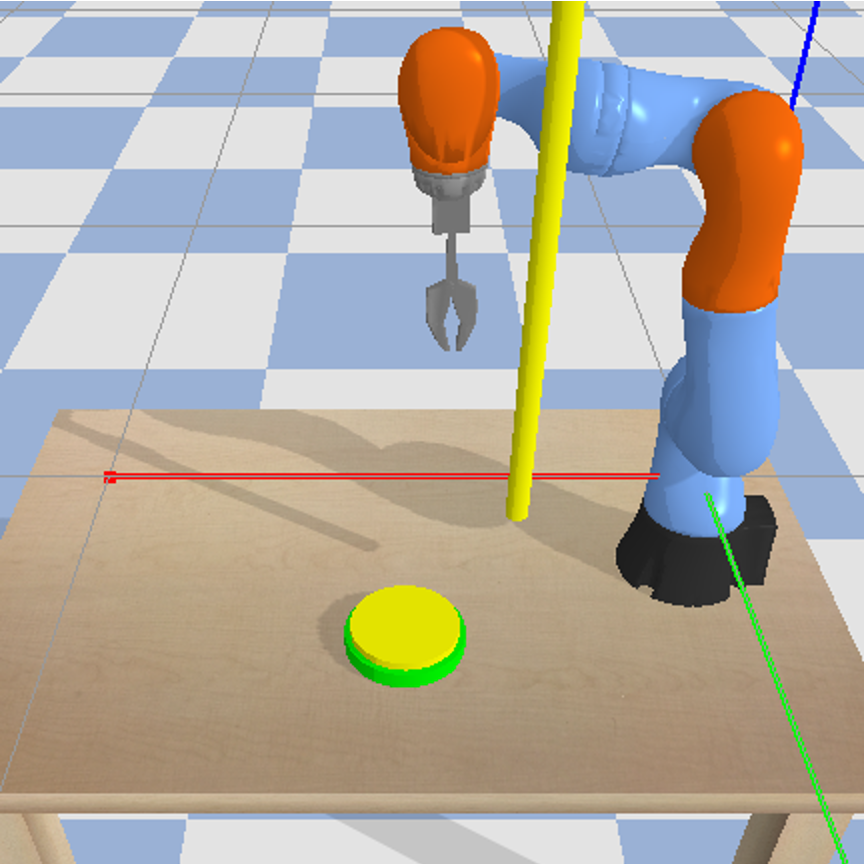}
        }
    \end{minipage}
    \hfill
    \begin{minipage}[t]{0.19\textwidth}
    \centering
    \subfigure[Kuka-Pick: Kuka picks up a fruit while not bumping into a moving cylinder.]{
        \includegraphics[width=0.86\textwidth]{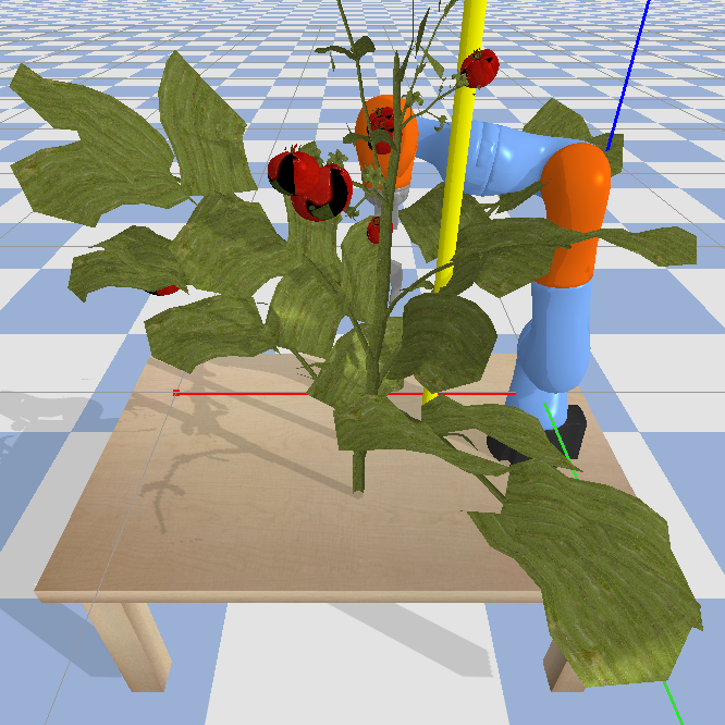}
    }
    \end{minipage}
    \hfill
    \begin{minipage}[t]{0.19\textwidth}
    \centering
    \subfigure[InMoov-Stretch: a humanoid InMoov reaches a fruit with a natural posture.]{
        \includegraphics[width=0.86\textwidth]{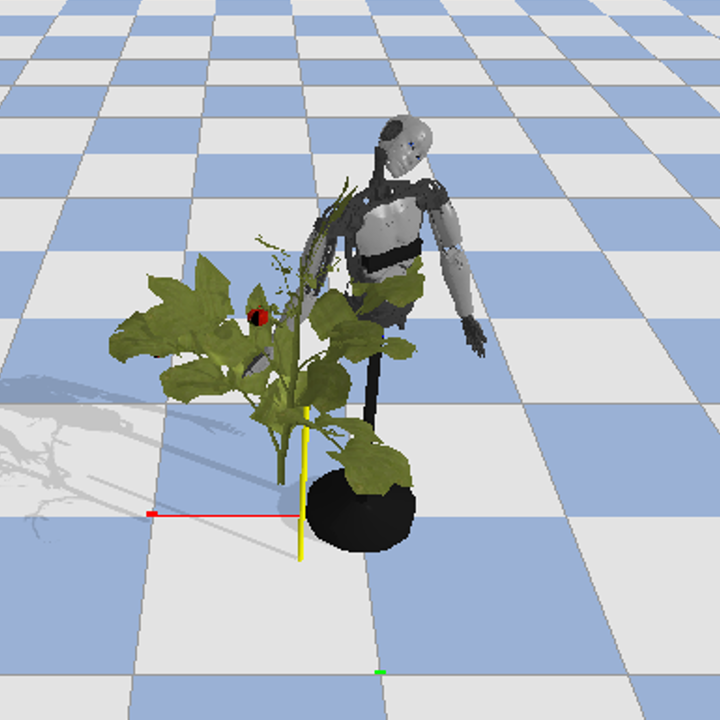}
    }
    \end{minipage}
    \hfill
    \begin{minipage}[t]{0.19\textwidth}
    \centering
    \subfigure[An ill-formed morphology to reach the fruit while violating safety constraint.]{
        \includegraphics[width=0.86\textwidth]{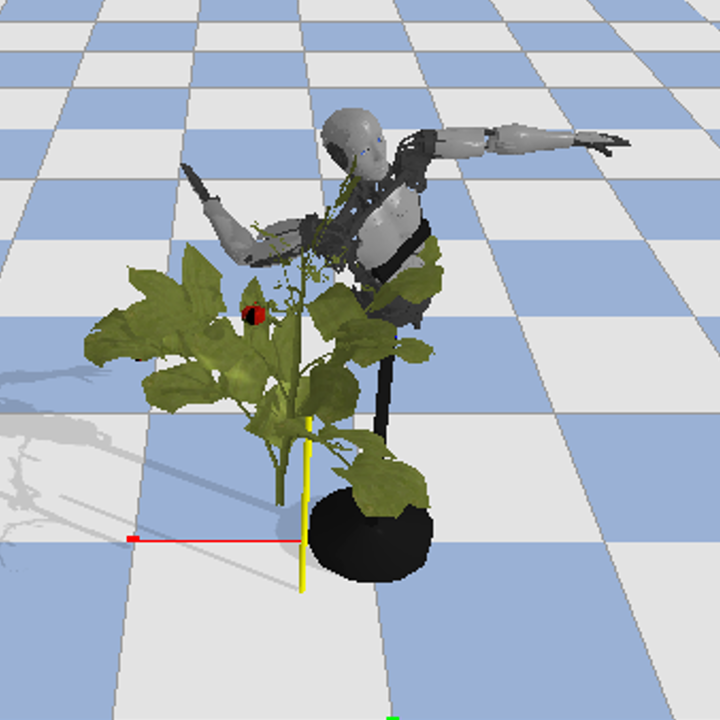}
    }
    \end{minipage}

    \caption{
    (a)-(d): Four simulated safe-critical tasks where we assess five safe RL algorithms;
    (e): An illustration of safety constraint violation.
    }
    \label{fig:sim_env}
\end{figure*}
\begin{figure*}[ht]
    \centering
    \begin{minipage}{0.4\textwidth}
    \centering
    \subfigure[Initial and end Kuka arm positions reaching the target while avoiding the cylinder obstacle in-between.]{
        \includegraphics[width=0.49\textwidth]{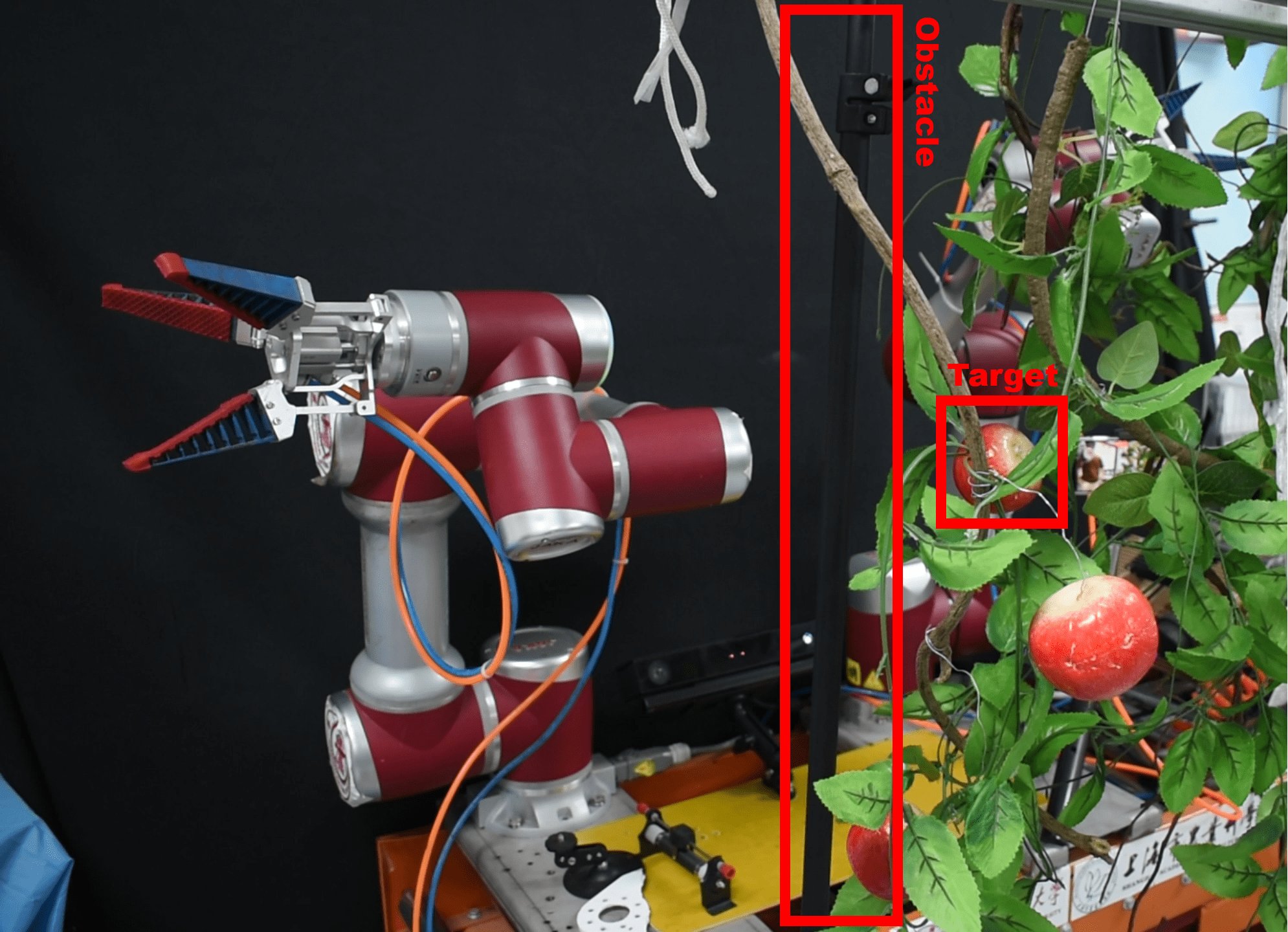}
        \hspace{0.2in}
        \includegraphics[width=0.49\textwidth]{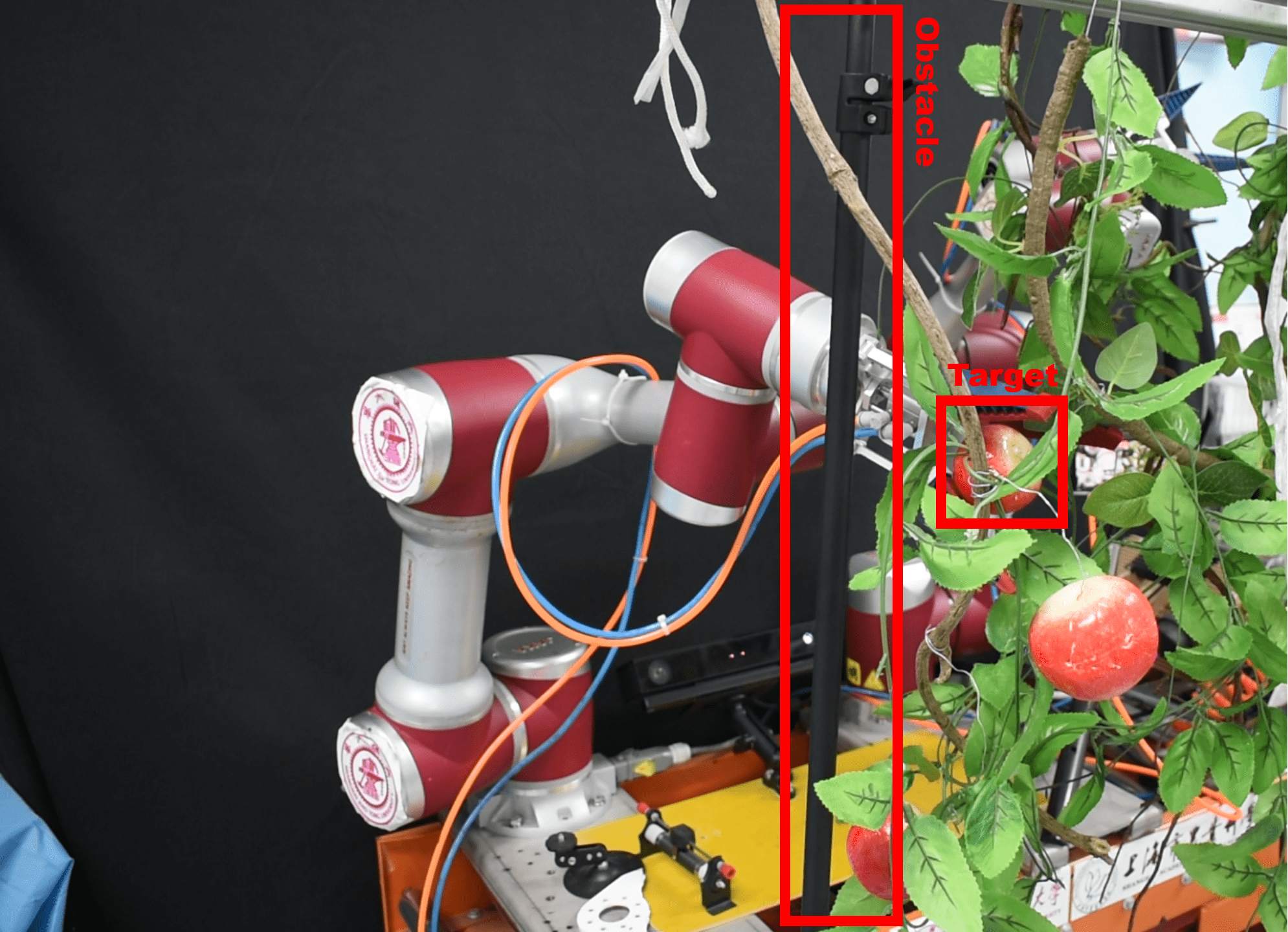}
    }
    \end{minipage}
    \hspace{0.5in}
    \begin{minipage}{0.4\textwidth}
    \centering
    \subfigure[Initial and end InMoov arm postures searching for a natural arm stretch trajectory to reach the target.]{
        \includegraphics[width=0.49\textwidth]{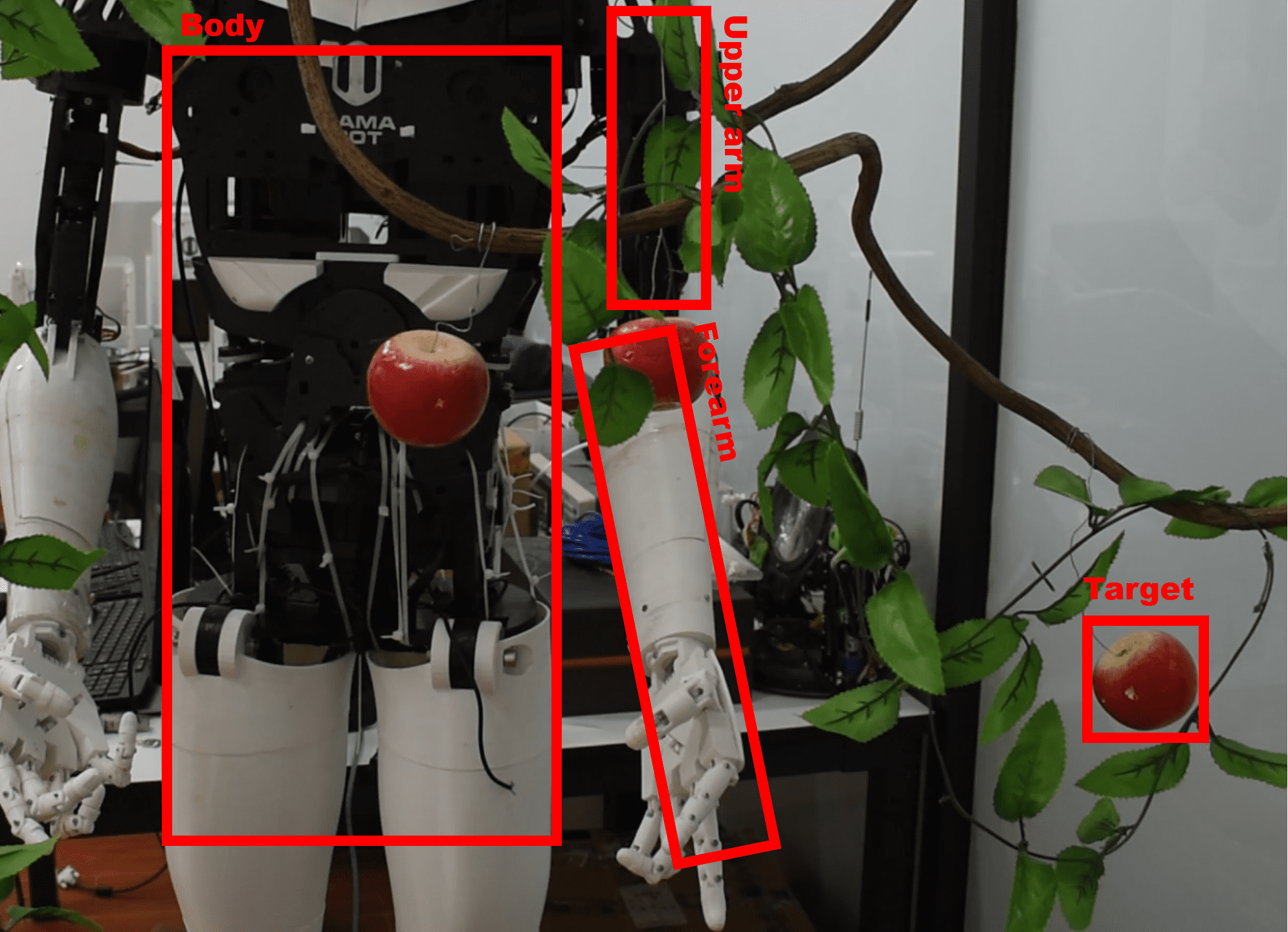}
        \hspace{0.2in}
        \includegraphics[width=0.49\textwidth]{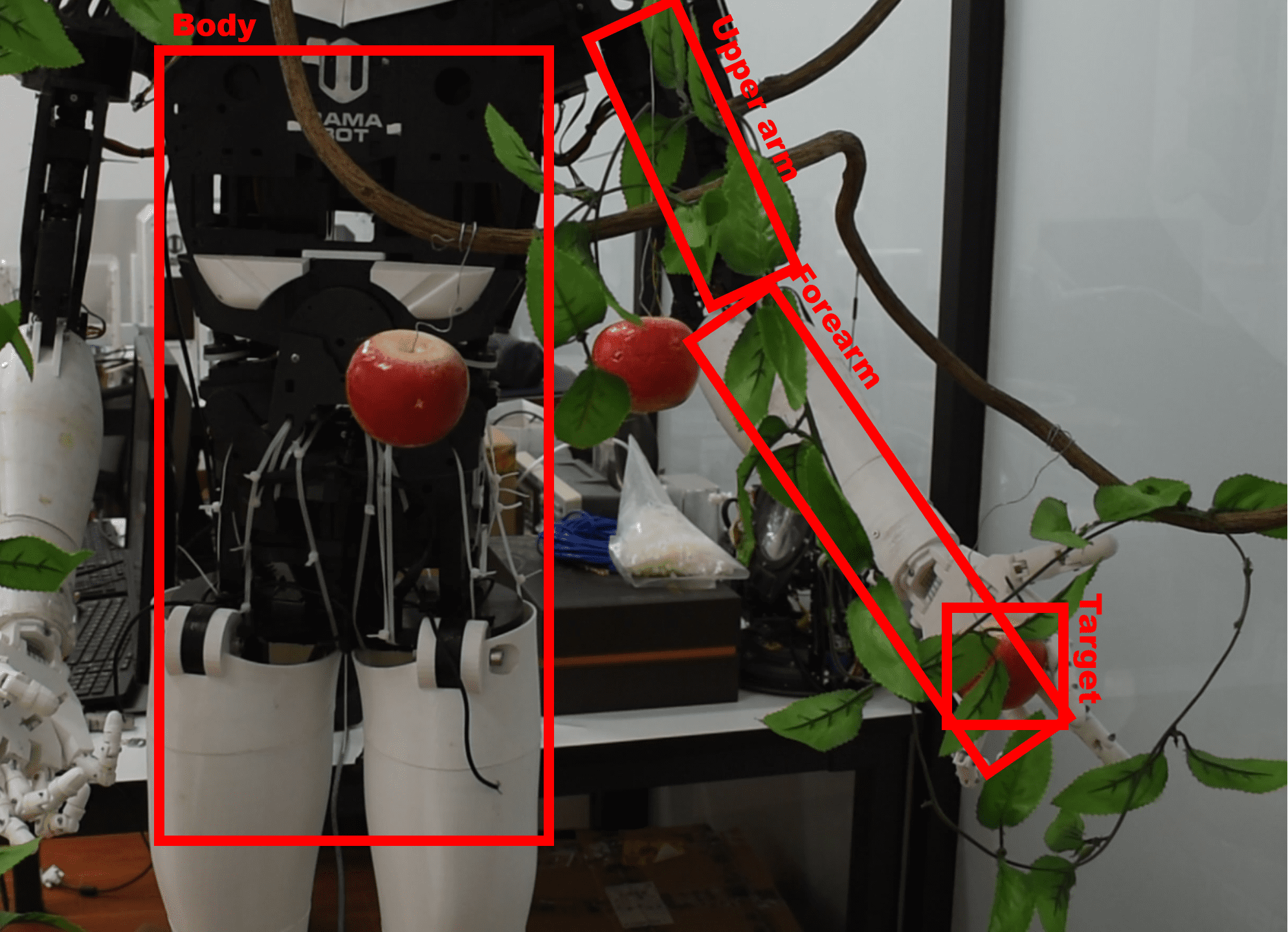}
    }
    \end{minipage}
    \hspace{0.35in}
    \caption{The initial and end pose of the robots in real-world \textit{Kuka-Pick} and
    \textit{InMoov-Stretch}. 
    }
    \label{fig:jaka_training}
\end{figure*}
\begin{figure*}[t]
    \centering
    \hspace{-0.1in}
    \begin{minipage}{0.2375\textwidth}
    \centering
    \subfigure[\textit{Ant-Run}]{
        \includegraphics[width=\textwidth]{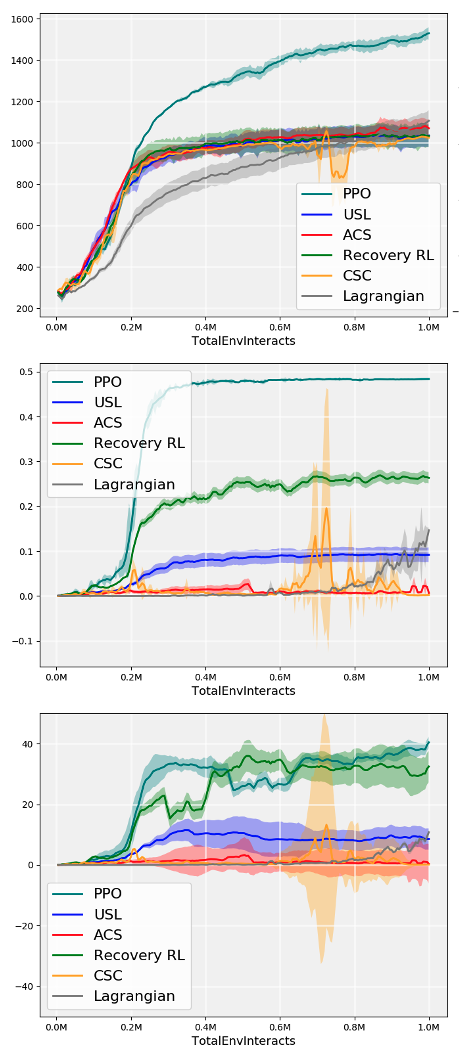}
    }
    \end{minipage}
    \hfill
    \begin{minipage}{0.230\textwidth}
    \centering
    \subfigure[\textit{Kuka-Reach}]{
        \includegraphics[width=\textwidth]{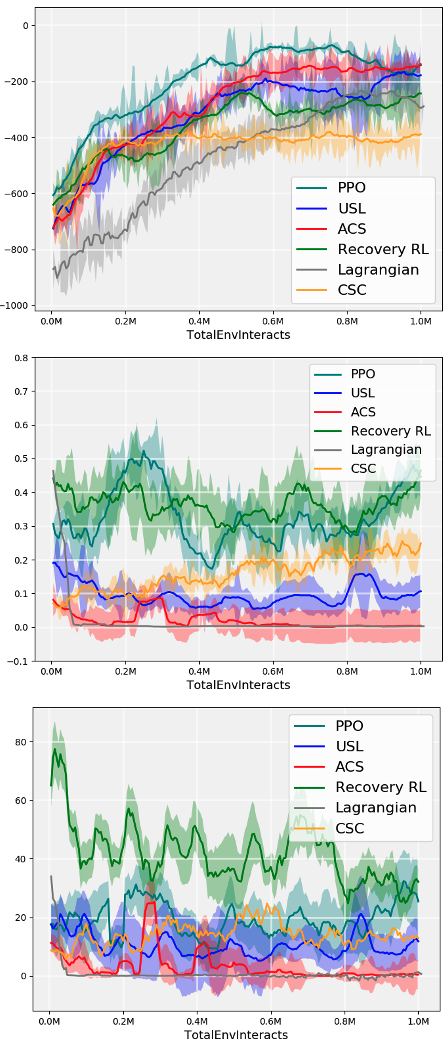}
    }
    \end{minipage}
    \hfill
    \begin{minipage}{0.232\textwidth}
    \centering
    \subfigure[\textit{Kuka-Pick}]{
        \includegraphics[width=\textwidth]{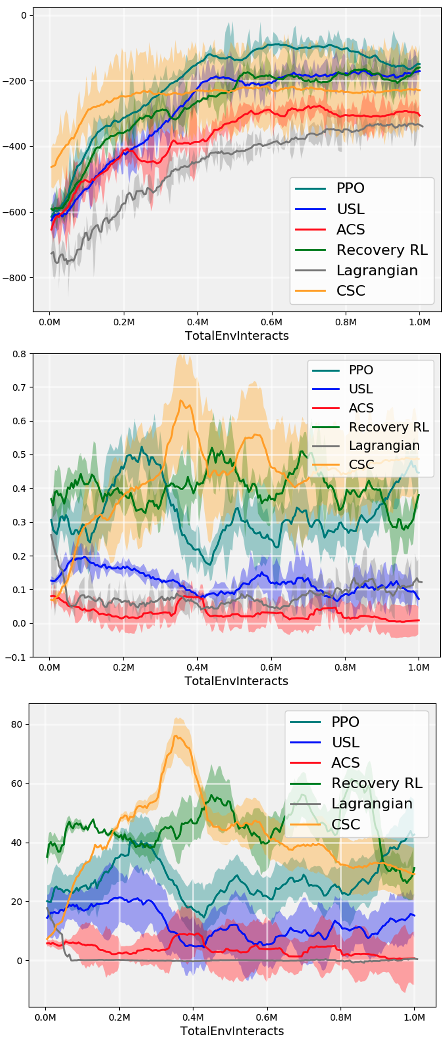}
    }
    \end{minipage}
    \hfill
    \begin{minipage}{0.23\textwidth}
    \centering
    \subfigure[\textit{InMoov-Stretch}]{
        \includegraphics[width=\textwidth]{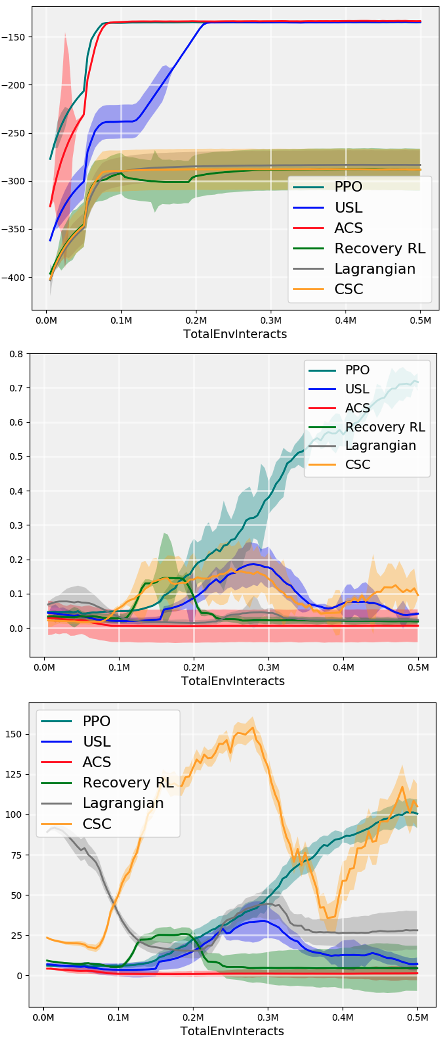}
    }
    \end{minipage}
    \hspace{0.1in}
    \caption{
    In-training curves of episodic return $J_r$ (top row), total cost rate $J_C$ (middle 
    row), and temporal safety cost rate $J_{TC}$ (bottom row) w.r.t. the number of 
    interactions of different algorithms on four safety-critical simulation tasks. 
    %
    }
    \label{fig:learning_curve}
    \vspace{-0.2in}
\end{figure*}

\bfsection{Fast ACS projection.}
The initial guess produced by the upper policy layer does not strictly satisfy 
\eqnref{eqn:sufficient_condition}, given that \eqnref{eqn:lagrangian} optimizes cumulative 
costs and $\lambda$ is hard to tune. 
Therefore, to strictly satisfy constraint in \eqnref{eqn:sufficient_condition}, we propose 
to employ L-BFGS~\cite{liu1989limited}, an efficient 
Quasi-Newton method to \textit{iteratively} correct $u_0$. 
L-BFGS is a memory-efficient method that updates control action by: 
\begin{equation}\label{eqn:bfgs_update}
    u^{k+1} = u^k - \eta * H_k^{-1} g_k
\end{equation}
where $\eta$ is the learning rate, $g_k = \frac{\delta}{\delta u^k}[A_C^{\pi_\theta}- 
\mathcal{F}(\alpha - V_C^{\pi_\theta})]$ is the gradient vector, and $H_k^{-1}$ is the 
approximation of the inversed Hessian matrix. 
BFGS finds a better projection axis and step-length through approximating an additional 
second-order Hessian inverse.
We show in \secref{sec:exp} that \algname can recover safety in few steps, even when 
against an adversarial policy.

We claim L-BFGS to be a better projection strategy in ACS for two reasons: 
1) safety-critical tasks usually require immediate response, and BFGS methods exhibit faster 
convergence rate by trading space for time; 
2) the optimal solution is embedded in the locally-convex sub-optimal region found by 
\eqnref{eqn:lagrangian}~\cite{donti2021dc3}, making it possible to correct the action in 
one step. 
In addition, it is insensitive to $\eta$~\cite{liu1989limited}, which differentiates from 
existing work~\cite{zhang2023evaluating} requiring tedious hyperparameters (e.g., learning 
rate) tuning. 
\begin{table*}[ht]
\centering
\caption{
Mean performance of 20 episodes at convergence on four safety-critical simulation tasks.
}
\vspace{-0.1in}
\setlength{\tabcolsep}{2pt}
\renewcommand{\arraystretch}{0}
\resizebox{0.9\linewidth}{!}{
\begin{tabular}{l | l | l l l l l l l}
\toprule
\multicolumn{2}{l}{Task}                       
&Safety Layer &PPO &Recovery RL &CSC &USL &PPO-Lagrangian &\algname (Ours) \\
\midrule
\multirow{3}{*}{\begin{tabular}[c]{@{}l@{}}\textit{Ant-run}\end{tabular}}  
&$J_r\uparrow$  
&$933.2_{\pm 12.98}$ &$\textbf{1512.7}_{\pm 32.10}$ &$1028.1_{\pm 87.05}$ &$1024.7_{\pm 11.68}$ 
&$981.7_{\pm 26.96}$ &$990.6_{\pm 81.58}$ &$1103.7_{\pm 50.10}$ \\
&$J_C\downarrow$
&$0.087_{\pm 0.01}$ &$0.48_{\pm 0.001}$ &$0.26_{\pm 0.00}$ & $\textbf{0.002}_{\pm 0.00}$ 
&$0.09_{\pm 0.01}$  &$0.11_{\pm 0.04}$ &$0.016_{\pm 0.00}$  \\
&$J_{TC}\downarrow$
&1.3e-3 &9e-4 &1e-3 &9.6e-3 &9e-3 &\textbf{8e-4} &6e-3 \\
\midrule
\multirow{3}{*}{\begin{tabular}[c]{@{}l@{}}\textit{Kuka-Reach}\end{tabular}} 
&$J_r\uparrow$
&-$256.9_{\pm 116.8}$ &-$141.8_{\pm 80.7}$ &-$243.1_{\pm 145.10}$ &-$388.3_{\pm 114.5}$ 
&-$178.7_{\pm 142.4}$ &-$289.1_{\pm 106.2}$  &$\textbf{-138.4}_{\pm 60.89}$   \\
&$J_C\downarrow$ 
&$0.24_{\pm 0.03}$  &$0.46_{\pm 0.06}$ &$0.44_{\pm 0.15}$ &$0.22_{\pm 0.05}$ 
&$0.05_{\pm 0.02}$  &$0.08_{\pm 0.001}$ &$\textbf{0.0032}_{\pm 0.04}$\\
&$J_{TC}\downarrow$   
&5.5e-3 &\textbf{1e-3} &6e-3 &5.4e-3 &1.4e-2 &\textbf{1e-3} &1e-2   \\
\midrule
\multirow{3}{*}{\begin{tabular}[c]{@{}l@{}}\textit{Kuka-Pick}\end{tabular}} 
&$J_r\uparrow$
&-$183.5_{\pm 51.6}$ &$\textbf{-149.7}_{\pm 81.2}$ &-$159.6_{\pm 45.4}$ &-$229.0_{\pm 247.3}$ 
&-$170.6_{\pm 66.8}$ &-$339.7_{\pm 64.0}$ &-$306.1_{\pm 57.9}$    \\
&$J_C\downarrow$ 
&$0.29_{\pm 0.03}$ &$0.47_{\pm 0.06}$ &$0.34_{\pm 0.05}$ &$0.49_{\pm 0.12}$ 
&$0.08_{\pm 0.03}$ &$0.13_{\pm 0.08}$ &$\textbf{0.007}_{\pm 0.08}$ \\
&$J_{TC}\downarrow$
&8.9e-3  &1e-3 &9.3e-3 &6e-3   &1.2e-3  &\textbf{9e-4} &9e-3 \\
\midrule
\multirow{3}{*}{\begin{tabular}[c]{@{}l@{}}\textit{InMoov-Stretch}\end{tabular}}
&$J_r\uparrow$  
&-$288.2_{\pm 44.7}$ &-$134.9_{\pm 0.2}$ &-$288.3_{\pm 43.5}$ &$288.2_{\pm 42.8}$ 
&-$134.6_{\pm 0.23}$ &-$283.3_{\pm 30.3}$  &$\textbf{-133.7}_{\pm 1.1}$ \\
&$J_C\downarrow$
&$0.08_{\pm 0.04}$ &$0.04_{\pm 0.01}$ &$0.02_{\pm 0.009}$ &$0.11_{\pm 0.05}$ 
&$0.04_{\pm 0.01}$ &$0.03_{\pm 0.009}$ &$\textbf{0.007}_{\pm 0.05}$ \\
&$J_{TC}\downarrow$
&8e-3  &\textbf{4.4e-3} &1.8e-2  &3.2e-2  &2.8e-2 &4.5e-3  &2e-2 \\
\bottomrule
\end{tabular}}
\label{tab:mean_performance}
\vspace{-0.1in}
\end{table*}
\begin{table*}[ht]
\centering
\caption{
    Quantitative results over 50 episodes on real-world tasks.
}
\vspace{-0.1in}
\setlength{\tabcolsep}{2pt}
\renewcommand{\arraystretch}{1}
\resizebox{0.8\linewidth}{!}{
\begin{tabular}{ l | c c | c c | c c | c c }
\toprule
\multirow{2}{*}{Method} 
&\multicolumn{2}{c|}{PPO}
&\multicolumn{2}{c|}{CSC}
&\multicolumn{2}{c|}{USL}
&\multicolumn{2}{c}{\algname (ours)} \\
&\textit{Kuka-Pick} &\textit{InMoov-Stretch} 
&\textit{Kuka-Pick} &\textit{InMoov-Stretch} 
&\textit{Kuka-Pick} &\textit{InMoov-Stretch} 
&\textit{Kuka-Pick} &\textit{InMoov-Stretch}   \\ 
\midrule
Success Rate$\uparrow$ 
&\textbf{86\%} &\textbf{64\%} &42\% &44\% &56\% &52\%   &62\% &50\% \\
Avg \# of Collisions$\downarrow$
&3.6 &1.0 &1.74 &0.62 &0.88 &0.48 &\textbf{0.24} &\textbf{0.28} \\
Avg \# of Iterations$\downarrow$
&0 &0 &4.0 &\textbf{4.6} &2.6 &9.0 &\textbf{1.9} &6.8 \\
\bottomrule
\end{tabular}
}
\label{tab:real_result}
\vspace{-0.2in}
\end{table*}
%
\section{Experiment}\label{sec:exp}
To thoroughly assess the effectiveness of \algname, we conduct experiments
on both simulated and real-world safety-critical tasks.
During simulations, we test on one speed-planning problem, \textit{Ant-Run}, in addition
to three manipulation tasks, including \textit{Kuka-Reach}, \textit{Kuka-Pick} and
\textit{InMoov-Stretch}.
In real-world scenarios, we implement the real-world versions of \textit{Kuka-Pick}
and \textit{InMoov-Stretch}, which share the same tasks and safety specifications as
their simulated counterparts, to further assess \algname's robustness in real-world settings. 
%

\subsection{Experimental Setup}
\bfsection{Simulated tasks.} 
As shown in \figref{fig:sim_env}, four tasks are designed to evaluate
\algname along with other safe RL methods.
We have:
\begin{list}{\labelitemi}{\leftmargin=0.7em}
    \item \textit{Ant-Run:} 
    utilizes a simple quadrupedal robot to assess both effectiveness and time efficiency
    by constraining the robot to run within a certain velocity limit.
    %
    \item \textit{Kuka-Reach:}
    employs a 7-DOF Kuka robot arm to assess the effectiveness of \algname during human-robot
    interaction.
    %
    As shown in \figref{fig:sim_env}(b), the robot arm is trying to reach for a 
    table button while avoiding collision with a static cylinder (yellow in the figure) 
    which represents the human. 
    \item \textit{Kuka-Pick:}
    extends \textit{Kuka-Reach} to include dynamic obstacles where, as shown in 
    \figref{fig:sim_env}(c), the robot arm seeks to pick a tomato on the tree while
    avoiding collisions with a moving cylinder (yellow) which represents 
    moving human and other tomatoes (static obstables) on the tree.
    %
    %
    \item \textit{InMoov-Stretch:}
    evaluates high-dimensional control generalizability and robustness.
    It utilizes an \textit{InMoov} humanoid which has 53 actively controllable 
    joints~\cite{zhao2019bionic}. 
    As shown in \figref{fig:sim_env}(d), the robot is trying to reach for the fruit 
    on the tree with a human-like arm stretch. 
    %
\end{list}


\bfsection{Real-world tasks.} 
To assess the robustness of \algname on real-world control tasks, we convert 
\textit{Kuka-Pick} and \textit{InMoov-Stretch} to their real-world counterparts
through point-cloud reconstruction and key-point matching~\cite{zhao2019bionic}. 
%
The task and safety specifications remain identical as in the simulations.
Examples from successful ACS episode runs of initial and end joint states for both tasks  
are shown in \figref{fig:jaka_training}. 

\bfsection{Baselines.}
Six SOTA safe RL approaches, including Safety Layer~\cite{dalal2018safe}, 
Proximal Policy Optimization (PPO)~\cite{schulman2017proximal}, 
Recovery RL~\cite{thananjeyan2021recovery},
Conservative Safe Critics (CSC)~\cite{bharadhwaj2020conservative}
Unrolling Safety Layer (USL)~\cite{zhang2023evaluating}, 
and PPO-Lagrangian~\cite{ray2019benchmarking}, are included in our experiments.
Official implementations and recommended parameters are used for each method.
%

\bfsection{Metrics.}
We first consider common metrics~\cite{ray2019benchmarking} such as episodic return $J_r$, 
total cost rate $J_C$, success rate, average number of collisions and average number of 
iterations for iterative methods such as USL, CSC, and ACS.
%
Besides, we argue that episodic inference time is also a critical dimension to consider 
coupled with safety performance. 
Therefore, we propose a novel metric, \textit{temporal cost rate}, which measures the safety 
performance along with forward inference time.
Specifically, we have
$J_{TC}=\frac{\text{accumulated cost}}{\text{length of the 
episode}}*\overline{t}_{\text{forward}}$,
%
where lower values indicate better and faster safety performance.

\subsection{Results}
In all experimental results, we set $\alpha$ to 0.2. 
More investigations on the impact of varying $\alpha$ are elaborated in 
\secref{subsec:ablation}.
%
%

\bfsection{Simulation tasks results.}
The results of in-training performance of all methods across four simulation tasks are 
shown in \figref{fig:learning_curve}. The corresponding numerical result is shown in \tabref{tab:mean_performance}.
We see that \algname achieves the best task performance on \textit{Ant-Run} (+11.2\%), 
\textit{Kuka-Reach} (+48.9\%) and \textit{InMoov-Stretch} (+47.9\%) while preserving 
nearly zero safety violation, indicating that \algname can quickly learn from failures 
and find a better trade-off boundary to balance between task optimality and safety.
%
%
We also notice from temperal cost rate $J_{TC}$ in \tabref{tab:mean_performance} that \algname 
is a faster projection method thus could cater better towards time-critical tasks.
%
While other algorithms may achieve better task objective in 
\textit{Kuka-Pick} ($-24.2\%$), they either fail to guarantee safety nor provide 
real-time control response. 
On the contrary, by implicitly predicting the trajectory of the obstacle with its 
advantage network and bounding it with an adaptive chance constraint, 
\algname achieves nearly zero safety violation on this task.
%

\bfsection{Real-world tasks results.}
We evaluate the performance of \algname against USL, CSC and PPO in terms of 
success rate, average number of collisions, and number of iterations to derive each 
action on both tasks.
As shown in \tabref{tab:real_result}, results demonstrate that \algname effectively 
adapts to stochastic real-world tasks and exhibit better success rate (+30\%) while 
preserving low safety violation (-65\%) with fewer iterations on both tasks. 
Since PPO achieves the best task objective via a brute-force path, it severely violates 
safety constraint. 
On the contrary, ACS outperforms all other methods in balancing the task and safety 
considerations.

\subsection{Recovery Capability against Adversarial Policy}
To further assess efficiency and robustness, we design an experiment on \textit{Kuka-Reach} 
where we train an adversarial policy to induce the agent towards unsafe regions with 
higher cost. 
During an episode of 1000 steps, the target policy and adversarial policy take turns 
and alternate every 100 steps.

As shown in \figref{fig:recoverability}(a), \algname demonstrates its effectiveness in quickly 
decreasing the safety cost to $0$ whenever the adversarial policy induces the agent to 
an unsafe state.
In comparison, USL fails to fully recover before the adversary comes to play again. 
More notably, as shown in \figref{fig:recoverability}(b), \algname quickly recovers 
even from the worst unsafe states (i.e., $cost = 1$) and remains safe during its vigilance, 
while the other algorithms keep degenerating under adversarial attacks.
\begin{figure}[h]
    \centering
    \begin{minipage}{0.233\textwidth}
    \centering
    \subfigure[Step-wise cost signal]{
        \includegraphics[width=\textwidth]{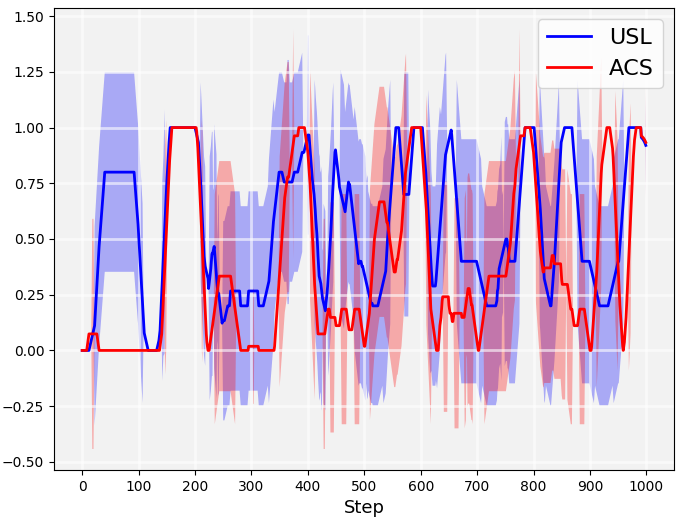}
    }
    \end{minipage}\hfill
    \begin{minipage}{0.233\textwidth}
    \centering
    \subfigure[Accumulated cost in an episode]{
        \includegraphics[width=\textwidth]{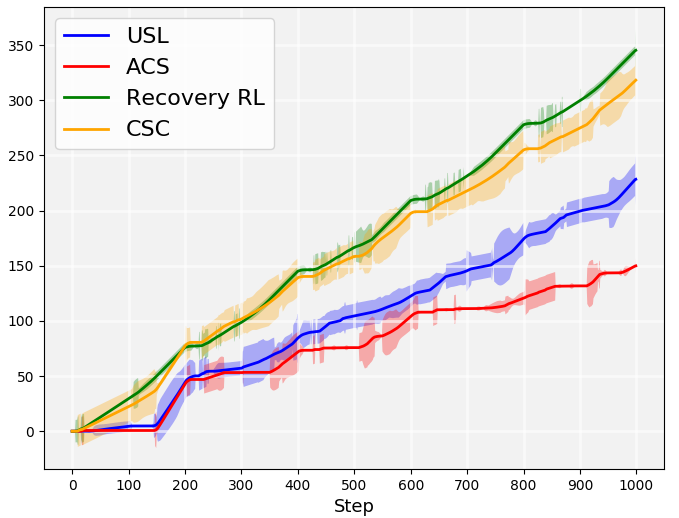}
    }
    \end{minipage}
    \caption{
    Recovery capabilities in terms of cost against an adversarial policy. 
    %
    }
    \vspace{-0.1in}
    \label{fig:recoverability}
\end{figure}

\subsection{\algname with Different Tolerance Thresholds}\label{subsec:ablation}
Additional experiments demonstrating how different tolerances $\alpha$ affect
the performance of \algname is further illustrated in \appref{subsec:ablation}.
Notably, even when tolerance $\alpha=1$, indicating that \algname depends solely on the 
sub-optimal policy layer without any projection, it still outperforms all competing 
methods. 
When tolerance $\alpha=0.2$ the controller can find the best trade-off between task 
and safety performance. 
\section{CONCLUSIONS}\label{sec:conclusions}
In this paper, we propose \textbf{A}daptive \textbf{C}hance-constraint \textbf{S}afeguard 
(\algname), a novel safe RL framework utilizing a hierarchical architecture to correct unsafe
actions yielded by the upper policy layer via a fast Quasi-Newton method. 
Through extensive theoretical analysis and experiments on both
simulated and real-world tasks, we demonstrate \algname's superiority in enforcing 
safety while preserving optimality and robustness across different scenarios. 
%
%

\clearpage
\section*{APPENDIX}
%
\subsection{Safety Probability Approximation via Value Function}\label{app:value_approx}
In this section, we draw from the Bellman equation~\cite{sutton2018reinforcement} and 
approximate the expected long-term chance-constrained safety probability $\Psi$ in 
\eqnref{eqn:psi_x_k} with a value network $V_C^\pi(x_k)$ updated via RL explorations. 
\begin{theorem}\label{thm:safe_probability}
    Let $r_{s_i} = 1\{\bigcap_{\tau \in \mathcal{T}(k)}x_j \in \mathcal{S}_C)\}$ be a 
    Bernoulli random variable that indicates joint safety constraint satisfaction, and 
    $r_{C_i} = 1 - r_{s_i}$ be the one-shot indicator of the complementary unsafe set, 
    the expected safe possibility at $x_k$ is $\Psi(x_k) = 1 - V_C^{\pi}(x_{k})$.
\end{theorem}
\begin{proof}
From~\cite{sutton2018reinforcement} we have
\begin{subequations}
\begin{align}
V_C^{\pi}(x_{k+1}) & = \mathbb{E}_{\tau \sim \pi_\theta} \left[ \sum_i \mathcal{J}_{C_i} \right]  \label{eqn:proof_1}\\
& = \frac{\sum_{u\sim \pi_\theta} \sum_{x} \sum_i{(r_{c_i}}+\gamma V_C^{\pi}(x_{k+1}))}{\sum_{u\sim \pi_\theta}\sum_{x}\sum_i 1} \\
& = \sum_{\{\tau:r_{c_i}(\tau_T)=1\}}\gamma^T P(\tau_i) * 1 \label{eqn:unsafe_prob}\\
& = 1 - \Psi(x_{k+1})
\end{align}
\end{subequations}
Notice that the RHS of \eqnref{eqn:unsafe_prob} is essentially the possibility of unsafe trajectory 
$P(\tau_i)$ weighted by $\gamma^T$. 
Therefore, the expected safe possibility at $x_k$ is $\Psi(x_k) = 1 - V_C^{\pi}(x_{k})$.
\end{proof}

\subsection{In-training Safety Certificate for Dynamic Policy}\label{app:dynamic_policy}
%
In this section, we provide the detail proof of \thmref{thm:constraint}.
For simplicity, here we consider only one safety constraint (i.e., $|C|=1$) and omit the 
subscript $i$. 
\begin{proof}
Since $\Psi(x_k)$ can be approximated by the value function of safety cost 
$V_C^\pi(x_k)$ (\thmref{thm:safe_probability}) and the advantage function 
$A^{\pi_\theta}_C(x_k, u_k)=\mathbb{E}
[V_C(x_{k+1}|x_k, \pi_{\theta}(x_k))]-V_C^{\pi_{\theta}}(x_k)$, 
we first certify safety convergence for a stationary policy $\pi_\theta$ based on 
forward invariance:
\begin{subequations}
\begin{align}
\mathbb{E}[V_C(x_{k+1})|x _k, \pi_{\theta}(x_k)]-V_C^{\pi_{\theta}}(x_k) &\le \mathcal{F}(\alpha - V_C^{\pi_{\theta}}(x_k)) \nonumber   \\
&\le \alpha - V_C^{\pi_{\theta}}(x_k) \label{eqn:bound_proof}   \\
 \Rightarrow \hspace{0.05in} \mathbb{E}[V_C(x_{k+1}|x_k, \pi_{\theta}(x_k))] 
&\le \alpha \label{eqn:bound_conclusion}
\end{align}
\end{subequations}
%
where \eqnref{eqn:bound_proof} is derived by convexity.
%
Since no system dynamics model is available in our settings, it is impossible to guarantee 
zero in-training safety without any failures\cite{bharadhwaj2020conservative}. 
%
%
Therefore, we derive an upper-bound for policy updates using the trust-region 
method~\cite{schulman2015trust} to ensure our safety certificate, 
as previously defined for the stationary policy case, \eqnref{eqn:bound_conclusion} 
holds under dynamic policy updates as well. 
%
The trust-region method constrains the update of policy parameters by using the total variation distance \(D_{\text{TV}} = D_{\text{TV}}(\pi_{\theta} \| \pi_{\theta_{\text{old}}})\) to ensure the new policy \(\pi_\theta\) does not deviate significantly from the old policy \(\pi_{\theta_{\text{old}}}\).
Therefore, following the trust-region derivations in~\cite{schulman2015trust} and~\cite{achiam2017constrained}, we start with the following inequality for the safety value function under the new policy:
\begin{equation}\label{eqn:safe_bound}
V_C^{\pi_{\theta}} - V_C^{\pi_{\theta_{\text{old}}}} \leq \frac{1}{1-\gamma} \mathbb{E}_{x \sim \rho_{\theta_{\text{old}}}, u \sim \pi_{\theta}}[A_C^{\pi_{\theta_{\text{old}}}}] + \beta D_{\text{TV}},
\end{equation}
where \(\beta\) is a positive coefficient that weighs the total variation distance \(D_{\text{TV}}\) in the safety cost inequality. Here, \(\beta = \frac{2\gamma \max |E_{\tau \sim \pi_\theta} A_C^{\pi_{\theta_{\text{old}}}}|}{1-\gamma}\), \(\tau\) denotes a trajectory resulting from policy \(\pi\), and \(\gamma\) is the discount factor that controls the expected convergence rate.

Assuming the safety certificate holds for the old policy \(\pi_{\theta_{\text{old}}}\), which implies that the expected advantage under any state-action pair does not exceed a threshold \(\alpha\), we have the following:
\begin{equation}\label{eqn:old_policy_safety}
\mathbb{E}_{x \sim \rho_{\theta_{\text{old}}}, u \sim \pi_{\theta}}[A_C^{\pi_{\theta_{\text{old}}}}] \leq \alpha.
\end{equation}
To maintain the safety constraint for the updated policy, we derive an upper bound for \(D_{\text{TV}}\) by rearranging equation Eqn.~(\ref{eqn:safe_bound}) and considering the safety constraint in Eqn.~(\ref{eqn:old_policy_safety}):
\begin{equation}\label{eqn:upper_Dtv}
D_{\text{TV}} \leq \frac{(1-\gamma)(\alpha - \mathbb{E}_{x \sim \rho_{\theta_{\text{old}}}, u \sim \pi_{\theta}}[A_C^{\pi_{\theta_{\text{old}}}}])}{\beta}.
\end{equation}
With the upper bound for policy updates \(D_{\text{TV}}\) in equation Eqn.~(\ref{eqn:upper_Dtv}), we can easily derive a bound for the updated safety cost value function in Eqn.~(\ref{eqn:safe_bound}). Finally we have essentially certificated the updated policy to remain within a safe region defined by the trust-region method and ensure Eqn.~(\ref{eqn:bound_conclusion}) is grounded for the case of dynamic policy update.

\end{proof}

\subsection{Ablation Study on Tolerance Level $\alpha$}\label{subsec:ablation}
\begin{table}[h]
\caption{Quantitative results over tolerance level $\alpha$ on \textit{Kuka-Reach}.}
\centering
\begin{tabular}{l|lll}
\toprule
\textit{Kuka-Reach}  &$J_r$   & $J_C$  & $J_{TC}$ (s)            \\
\midrule
ACS ($\alpha=1)$   & -221.9 $\pm$ 87.40 & 0.05 $\pm$ 0.01 & 0                \\
ACS ($\alpha=0.6$) & -178.4 $\pm$ 85.20 & 0.02 $\pm$ 0.01 & 44e-4                 \\
ACS ($\alpha=0.2$) & -138.4 $\pm$ 60.89  & 0.0032 $\pm$ 0.04   & 1e-2\\
ACS ($\alpha=0.05$) & -196.3 $\pm$ 70.00 & 0.0028 $\pm$ 0.02   & 31e-3  \\
\bottomrule
\end{tabular}
\label{tab:alpha_ablation}
\end{table}


\section*{ACKNOWLEDGMENT}
The authors would like to thank Mahsa Ghasemi, Siddharth Gangadhar, Yorie Nakahira, 
Weiye Zhao, Changliu Liu for their precious comments and suggestions. Liang Gong would 
also like to thank JAKA Robotics for providing with the hardware for experimentation in this 
paper.


\bibliographystyle{plain}
\bibliography{ref}

\end{document}